\documentclass{llncs}

\usepackage{graphicx}
\usepackage{amsmath}
\usepackage{amssymb}
\usepackage{makecell}
\usepackage{color}
\usepackage{algorithm}
\usepackage{algorithmic}

\newcommand{\E}{\mathbb{E}}
\newcommand{\EE}[1]{\mathbb{E} \left[ {#1} \right]}
\newcommand{\X}{(X)}

\newcommand{\bx}{(\boldsymbol{x})}
\newcommand{\bK}{\boldsymbol{K}}
\newcommand{\prevF}[1]{f_{i-1} #1 , \cdots , f_1 #1}
\newcommand{\f}[1]{\boldsymbol{f}_{#1}}

\begin{document}

\title{Non-redundant Spectral Dimensionality Reduction}

\author{Yochai Blau and Tomer Michaeli}
\institute{Technion--Israel Institute of Technology, Haifa, Israel
\email{\{yochai@campus,tomer.m@ee\}.technion.ac.il}}

\maketitle

\begin{abstract}
Spectral dimensionality reduction algorithms are widely used in numerous domains, including for recognition, segmentation, tracking and visualization. However, despite their popularity, these algorithms suffer from a major limitation known as the ``repeated Eigen-directions'' phenomenon. That is, many of the embedding coordinates they produce typically capture the same direction along the data manifold. This leads to redundant and inefficient representations that do not reveal the true intrinsic dimensionality of the data. In this paper, we propose a general method for avoiding redundancy in spectral algorithms. Our approach relies on replacing the orthogonality constraints underlying those methods by unpredictability constraints. Specifically, we require that each embedding coordinate be unpredictable (in the statistical sense) from all previous ones. We prove that these constraints necessarily prevent redundancy, and provide a simple technique to incorporate them into existing methods. As we illustrate on challenging high-dimensional scenarios, our approach produces significantly more informative and compact representations, which improve visualization and classification tasks.
\end{abstract}

\section{Introduction}
The goal in nonlinear dimensionality reduction is to construct compact representations of high dimensional data, which preserve as much of the variability in the data as possible. Such techniques play a key role in diverse applications, including recognition and classification \cite{he2005face,belkin2004semi,geng2005supervised,pang2005face}, tracking \cite{lim2006dynamic,wang2003learning,lee2007modeling}, image and video segmentation \cite{zhang2006manifold,pless2003image,isola2014crisp}, pose estimation \cite{elgammal2004inferring,raytchev2004head,souvenir2008learning}, age estimation \cite{guo2008image},  spatial and temporal super-resolution \cite{chang2004super,pless2003image,chakrabarti2007super}, medical image and video analysis \cite{georg2008simultaneous,souvenir2006image,brun2003coloring} and data visualization \cite{vlachos2002non,lim2003planar,zhang2013trace,gisbrecht2015data}.

Many of the dimensionality reduction methods developed in the last two decades are based on spectral decomposition of some data-dependent (kernel) matrix. These include, e.g.\@, Locally Linear Embedding (LLE) \cite{roweis2000nonlinear}, Laplacian Eigenmaps (LEM) \cite{belkin2003laplacian}, Isomap \cite{tenenbaum2000global}, Hessian Eigenmaps (HLLE) \cite{donoho2003hessian}, Local Tangent Space Alignment (LTSA) \cite{zhang2004principal}, Diffusion Maps (DFM) \cite{coifman2006diffusion},  and Kernel Principal Component Analysis (KPCA) \cite{scholkopf1997kernel}. Methods in this family differ in how they construct the kernel matrix, but in all of them the eigenvectors of the kernel serve as the low-dimensional embedding of the data points \cite{ham2004kernel,bengio2004out,van2009dimensionality}.

A significant shortcoming of spectral dimensionality reduction algorithms is the ``repeated eigen-directions" phenomenon \cite{gerber2007robust,goldberg2008manifold,dsilva2015parsimonious}. That is, successive eigenvectors tend to represent directions along the data manifold which were already captured by previous ones. This leads to redundant representations that are unnecessarily larger than the intrinsic dimensionality of the data. To illustrate this effect, Fig.~\ref{fig:swissRoll} visualizes the two dimensional embeddings of a Swiss roll, as obtained by several popular spectral dimensionality reduction algorithms. It can be seen that in all the examined methods, the second dimension of the embedding carries no additional information with respect to the first. Specifically, although the first dimension already completely characterizes the position along the long axis (angular direction) of the manifold, the second dimension is also a function of this axis. Progression along the short axis (vertical direction) is captured only by the third eigenvector in this case (not shown). Therefore, the representation we obtain is 50\% redundant: Its second feature is a deterministic function of the first and thus superfluous.

\begin{figure}[!t]
	\begin{center}
		\includegraphics[width=\linewidth]{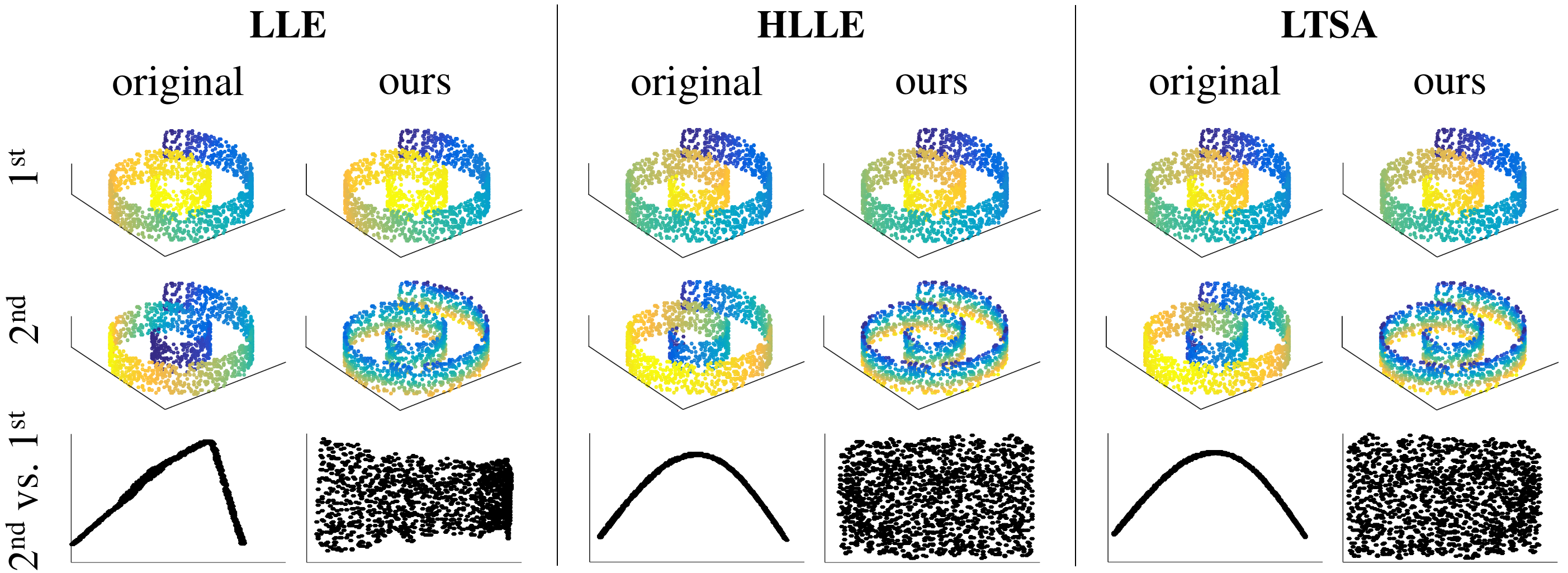}
	\end{center}
	\caption{The first two projections of data points lying on a Swiss roll manifold, as obtained with the original LLE, HLLE and LTSA algorithms and with our non-redundant versions of those algorithms. \emph{Top row:} The points colored by the first projection. \emph{Middle row:} The points colored by the second projection. As can be seen, the original algorithms redundantly capture progression along the angular direction twice. In contrast, in our versions of those algorithms, the second projection captures the vertical direction. \emph{Bottom row:} Scatter plot of the $2$nd projection vs.\@ the $1$st. In the original algorithms, the $2$nd projection is a function of the $1$st, while in our algorithms it is not.}
	\label{fig:swissRoll}
\end{figure}

In fact, the redundancy of spectral methods can be arbitrarily high. To see this, consider for example the embedding obtained by the LEM method, whose kernel approximates the Laplace-Beltrami operator on the manifold. The Swiss-roll corresponds to a two dimensional strip with edge lengths $L_1$ and $L_2$. Thus, the eigenfunctions and eigenvalues (with Neumann boundary conditions) are given in this case by
\begin{align}\label{eq:LB-eigenfunctions}
\phi_{k_1k_2}\left(x_1,x_2\right)&=\cos \left( \frac{k_1\pi x_1}{L_1} \right) \cos \left( \frac{k_2\pi x_2}{L_2} \right) \enspace ,\\
\lambda_{k_1k_2}&=\left( \frac{k_1\pi}{L_1} \right)^2 + \left( \frac{k_2\pi}{L_2} \right)^2 \enspace ,
\end{align}
for $k_1,k_2 = 0,1,2,\ldots$,  where $x_1$ and $x_2$ are the coordinates along the strip. Ignoring the trivial function ${\phi_{0,0}(x_1,x_2)=1}$, it can be seen that the first $\lfloor L_1/L_2 \rfloor$ eigenfunctions (corresponding to the smallest eigenvalues) are functions of only $x_1$ and not $x_2$ (see Fig.~\ref{fig:strips}). Thus, at least $\lfloor L_1/L_2 \rfloor + 1$ projections are required to capture the two dimensions of the manifold, which leads to a very inefficient representation when $L_1$ is much larger than $L_2$. In fact, projections $2,\ldots,\lfloor L_1/L_2 \rfloor$ are all functions of projection 1, and are thus redundant. For example, when $L_1 > 2 L_2$, the first two eigenfunctions are $\phi_{1,0}(x_1,x_2) = \cos ( \pi x_1 / L_1 )$ and $\phi_{2,0}(x_1,x_2) = \cos ( 2 \pi x_1 / L_1 )$, which clearly satisfy $\phi_{2,0}(x_1,x_2) = 2 \phi_{1,0}^2(x_1,x_2) - 1$. Notice that this redundancy appears despite the fact that the functions $\{\phi_{k_1k_2}\}$ are orthogonal (being eigenfunctions of a self-adjoint positive definite operator). This highlights the fact that \emph{orthogonality does not imply non-redundancy}.

\begin{figure}[!t]
	\begin{center}
		\includegraphics[width=0.7\linewidth]{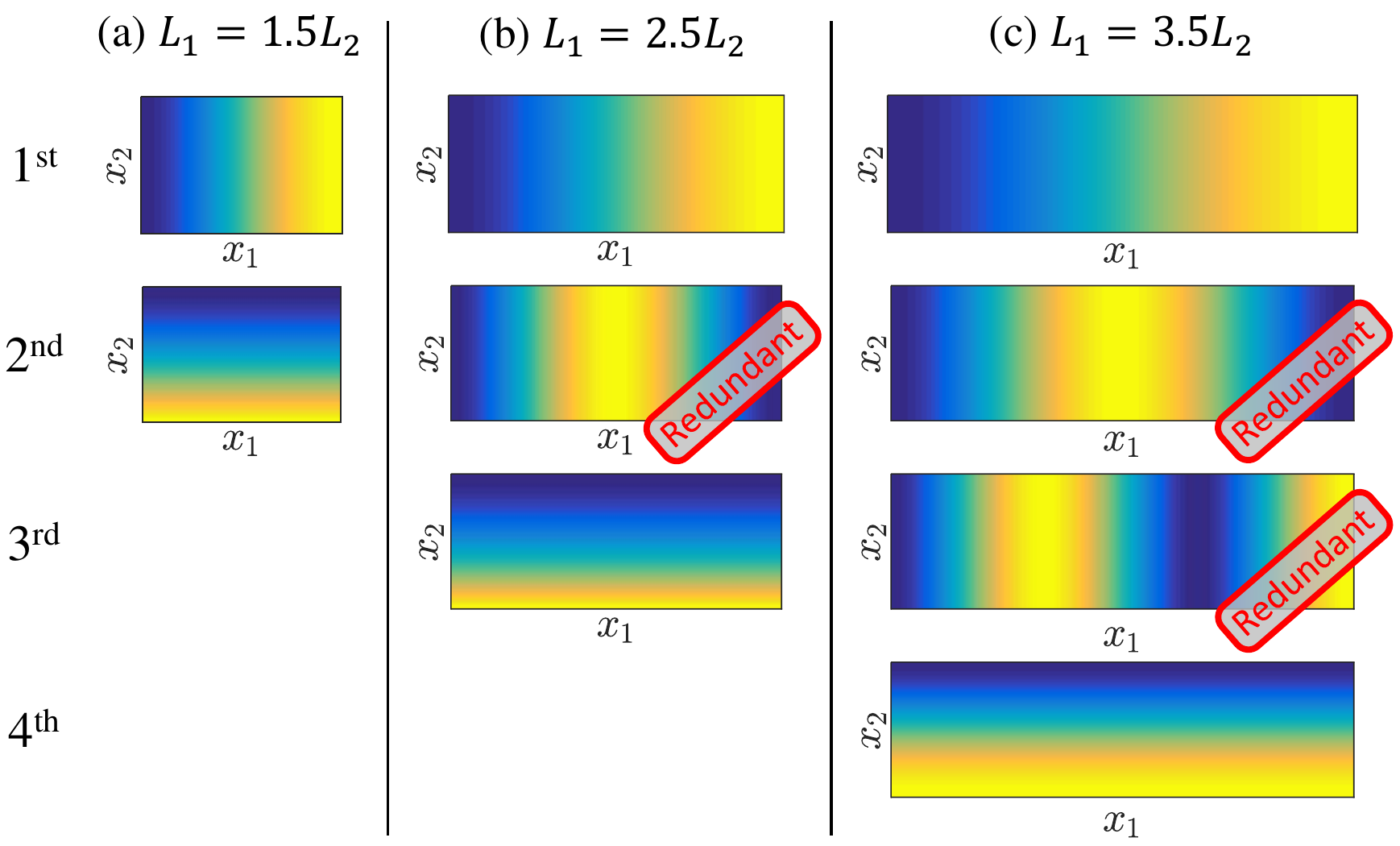}
	\end{center}
	\caption{A 2D strip with edge lengths (a)~$L_1=1.5L_2$, (b)~$L_1=2.5L_2$ and (c)~$L_1=3.5L_2$, colored according to the first few coordinates of the Laplacian Eigenmaps embedding (the eigenfunctions of the Laplace-Beltrami operator). Coordinates $2,\ldots,\lfloor L_1/L_2 \rfloor$ are redundant as they are all functions of only~$x_1$, which is already fully represented by the first coordinate.}
	\label{fig:strips}
\end{figure}

The above analysis is not unique to the LEM method. Indeed, as shown in~\cite{goldberg2008manifold}, spectral methods produce redundant representations whenever the variances of the data points along different manifold directions vary significantly. This observation, however, cannot serve to solve the problem as in most cases the underlying manifold is not known a-priori.

In this paper, we propose a general framework for eliminating the redundancy caused by repeated eigen-directions. Our approach applies to all spectral dimensionality reduction algorithms, and is based on replacing the orthogonality constraints underlying those methods, by unpredictability ones. Namely, we restrict subsequent projections to be unpredictable (in the statistical sense) from all previous ones. As we show, these constraints guarantee that the projections be non-redundant. Therefore, once a manifold dimension is fully represented by a set of projections in our method, the following projections must capture a new direction along the manifold. As we demonstrate on several high-dimensional data-sets, the embeddings produced by our algorithm are significantly more informative than those learned by conventional spectral methods.

\section{Related Work} \label{sec:RelatedWork}

\begin{figure}[!t]
	\begin{center}
		\includegraphics[width=\linewidth]{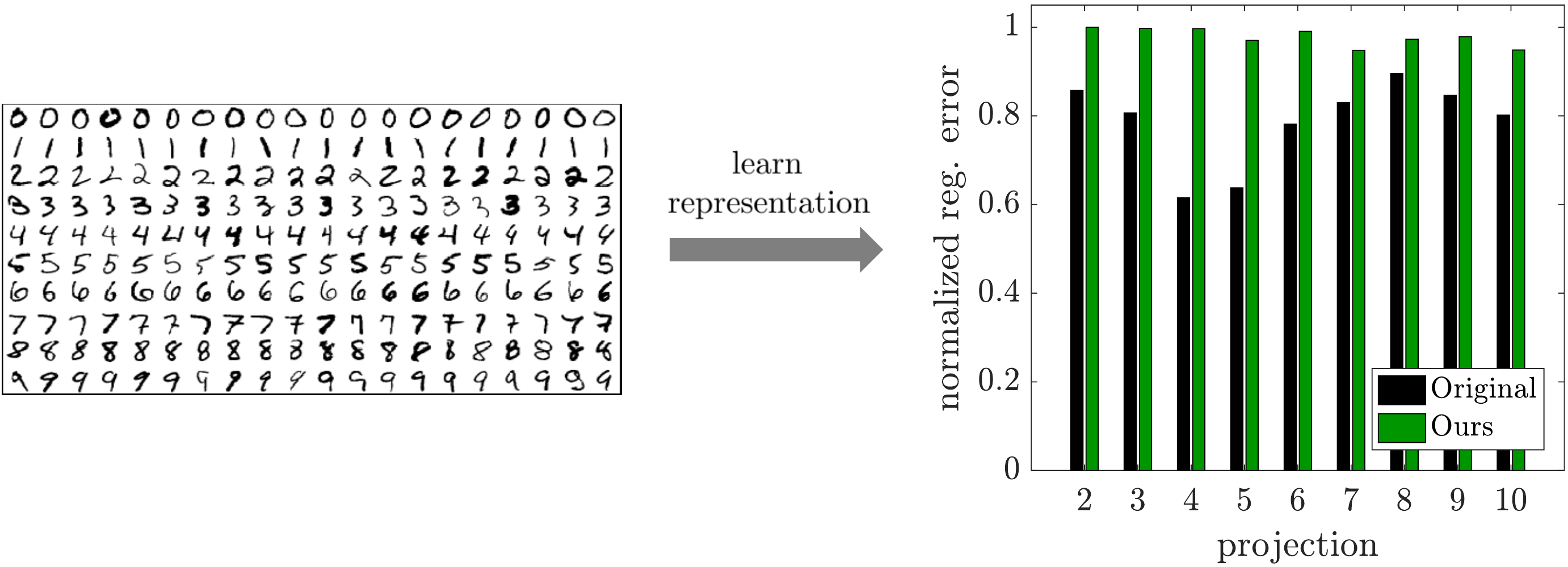}
	\end{center}
	\caption{A 10-dimensional representation of 15K MNIST handwritten digits \cite{MNIST} was learned with LEM and our non-redundant LEM. The normalized error attained by regressing each projection against all previous ones indicates to what extent the projection is redundant (higher is less redundant)  \cite{dsilva2015parsimonious}.}
	\label{fig:projRegression}
\end{figure}

Very few works suggested ways to battle the repeated eigen-directions phenomenon. Perhaps the simplest approach is to identify the redundant projections in a post-processing manner \cite{dsilva2015parsimonious}. In this method, one begins by computing a large set of projections. Each projection is then regressed against all previous ones (using some nonparametric regression method). Projections with low regression errors (i.e.\ which can be accurately predicted from the preceding ones) are discarded. This approach is quite efficient but usually works well only in simple situations. Its key limitation is that it is restricted to choose the projections from a given finite set of functions, which may not necessarily contain a ``good'' subset. Indeed, as we demonstrate in Fig.~\ref{fig:projRegression}, in real-world high-dimensional settings all the projections tend to be partially predictable from previous ones. Yet, there usually does not exist any single projection which can be considered fully redundant. Therefore, despite the obvious dependencies, almost no projection is practically discarded in this approach. In contrast, our algorithm produces projections which cannot be predicted from the previous ones (with normalized regression errors $\sim$100\%). Therefore, we are able to preserve more information about the data.

Another simple approach is to compute the projections sequentially, by eliminating the variations in the data which can be attributed to the projections that have already been computed. A naive way of doing so, would be to subtract from the data points their predictions based on all the previous projections. However, perhaps counter-intuitively, this \emph{sequential regression} process does not necessarily prevent redundancy. This is because the data points may fall off the manifold during the iterations, as demonstrated in Fig.~\ref{fig:ringExample}(b).

\begin{figure}[!t]
	\begin{center}
		\includegraphics[width=\linewidth]{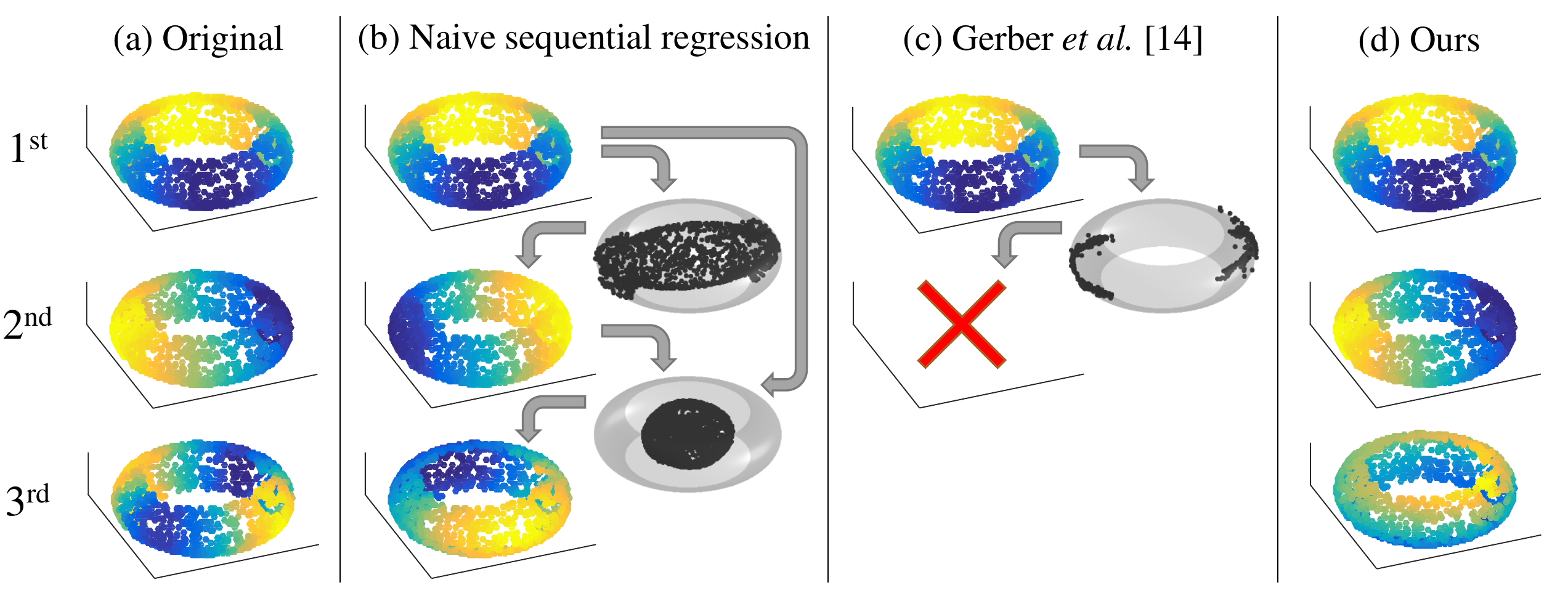}
	\end{center}
	\caption{(a)~The first three projections of points lying on a ring manifold, obtained with the original LEM algorithm. The projections correspond to $\cos(\theta)$, $\sin(\theta)$ and $\sin(2\theta+c)$, where $\theta$ is the outer angle of the ring. In this case, Projection~$2$ is not a function of Projection~$1$ and is thus non-redundant. But Projection~$3$ is a function of Projections $1$ and $2$, and is thus redundant. (b)~The projections obtained with the naive sequential regression approach (Sect.~\ref{sec:RelatedWork}). Here, Projection~$3$ is still redundant. The right column shows the points after subtracting their prediction from previous projections, which causes them to fall off the manifold. (c)~The projections obtained with the algorithm of~\cite{gerber2007robust}. Here, the algorithm halts after one projection. The right column shows the points after the advection process along the manifold, which results in two clusters forming an unconnected graph. (d)~The projections obtained with our non-redundant version of LEM. Our algorithm extracts a non-redundant third projection, which captures progression along the inner angle of the ring.}
	\label{fig:ringExample}
\end{figure}

A more sophisticated approach, suggested by Gerber et. al.\ \cite{gerber2007robust}, is to collapse the data points \emph{along the manifold} in the direction of the gradient of the previous projection. In this approach, the points always remain on the manifold. However, this method fails whenever a projection is a non-monotonic function of some coordinate along the manifold. This happens, for example, in the ring manifold of Fig.~\ref{fig:ringExample}. In this case, the first projection extracted by LEM corresponds to $\cos(\theta)$, where $\theta$ is the outer angle of the ring. Therefore, before computing the second projection, the advection process moves the points along the $\theta$ coordinate towards the locations at which $\cos(\theta)$ attains its mean value, which is~$0$. This causes the points with $\theta\in(0,\pi)$ to collapse to $\theta=\pi/2$, and the points with $\theta\in(\pi,2\pi)$ to collapse to $\theta=3\pi/2$. The two resulting clusters form an unconnected graph, so that LEM cannot be applied once more. An additional drawback of this method is that it requires a-priori knowledge of the manifold dimension. Furthermore, it is very computationally intensive and thus impractical for high-dimensional big data applications.

In this paper, we propose a different approach. Similarly to the methods described above, our algorithm is sequential. However, rather than heuristically modifying the data points in each stage, we propose to directly incorporate constraints which guarantee that the projections are not redundant.

\section{Eliminating Redundancy} \label{sec:Redundancy}

Nonlinear dimensionality reduction algorithms seek a set of \emph{non-linear} projections $f_i:\mathbb{R}^{D}\rightarrow\mathbb{R},\,i=1,\cdots,d$ which map $D$-dimensional data points $\boldsymbol{x}_n\in \mathbb{R}^D$ into a $d$-dimensional feature space ($d<D$).

\begin{definition}\label{def:redundancy}
	We call a sequence of projections $\{f_i\}$ \textbf{non-redundant} if none of them can be expressed as a function of the preceding ones. That is, for every~$i$,
	\begin{equation}\label{eq:redundantVector}
	f_i \bx \ne g (\prevF{ \bx })
	\end{equation}
	for every function $g:\mathbb{R}^{i-1}\rightarrow\mathbb{R}$.
\end{definition}

Let us see why existing spectral dimensionality reduction algorithms do not necessarily yield non-redundant projections. Spectral algorithms obtain the $i$th projection of all the data points, denoted by $\f{i} = ( f_i(\boldsymbol{x}_1) , \cdots , f_i(\boldsymbol{x}_N) )^T$, as the solution to the optimization problem\footnote{Note that LEM and DFM use slightly different constraints (see supplementary material). Also, note that some methods (e.g.\ LEM, LLE) rather \emph{minimize} the objective in (\ref{prob:dimReduction}). Here we address only the maximization problem, as minimizing $\f{i}^T \bK \f{i}$ is equivalent to maximizing $\f{i}^T \check{\bK} \f{i}$, where $\check{\bK} = \lambda_{\max}\boldsymbol{I}-\bK$ with $\lambda_{\max}$ denoting the largest eigenvalue of $\bK$ \cite{ham2004kernel,bengio2004out}.}
\begin{equation}\label{prob:dimReduction}
\begin{aligned}
& \max_{\f{i}}
& & \f{i}^T \bK \f{i} \\
& \hspace{0.085cm} \mbox{s.t.}
& & \boldsymbol{1}^T \f{i}=0\\
&&& \f{i}^T \f{i} = 1 \\
&&& \f{i}^T \f{j} =0,\quad \forall j<i \enspace .
\end{aligned}
\end{equation}
Here, $\boldsymbol{1}$ is an $N\times 1$ vector of ones and $\bK$ is an $N\times N$ algorithm-specific positive definite (kernel) matrix, constructed from the data points \cite{goldberg2008manifold,van2009dimensionality}. The first constraint in Problem~\eqref{prob:dimReduction} ensures that the projections have zero means. The last two constraints restrict the projections to have unit norms and to be orthogonal w.r.t. one another. The solution to Problem (\ref{prob:dimReduction}) is given by the $d$ top eigenvectors of the centered kernel matrix $(\boldsymbol{I}-\frac{1}{N}\boldsymbol{1}\boldsymbol{1}^T)\bK(\boldsymbol{I}-\frac{1}{N}\boldsymbol{1}\boldsymbol{1}^T)$. When $\bK$ is a stochastic matrix (e.g.\ LLE, LEM), the solution is simply eigenvectors $2,\ldots,d+1$ of $\bK$ (without centering).

The orthogonality constraints in Problem (\ref{prob:dimReduction}) guarantee that the projections be linearly independent. However, these constraints do not guarantee non-redundancy. To see this, it is insightful to interpret them in statistical terms. Assume that the data points $\{\boldsymbol{x}_n\}$ correspond to independent realizations of some random vector $X$. Then orthogonality corresponds to zero statistical correlation, as
\begin{equation} \label{eq:orthogonality}
\EE{f_i \X f_j \X} \approx \tfrac{1}{N} \sum_n f_i ( \boldsymbol{x}_n ) f_j ( \boldsymbol{x}_n ) = \tfrac{1}{N} \f{i}^T \f{j} = 0 \enspace .
\end{equation}
Therefore, in particular, the orthogonality constraints in \eqref{prob:dimReduction} guarantee that each projection be uncorrelated with any linear combination of the preceding projections. This implies that
$f_i \X$ \emph{cannot be a linear function of the previous projections} $\{f_j \X \}_{j<i}$. However, this does not prevent $f_i \X$ from being a \emph{nonlinear} function of the previous projections, which would lead to redundancy, as we saw in Figs.~\ref{fig:swissRoll},~\ref{fig:strips} and \ref{fig:ringExample}.

To enforce non-redundancy, i.e. each projection is not a function of the previous ones, we propose to use the following observation.
\begin{lemma}\label{lem:redundancy}
	A sequence of non-trivial zero-mean projections $\{f_i\}$ is \textbf{non-redundant} if each of them is \textbf{unpredictable} from the preceding ones, namely
	\begin{equation}\label{eq:unpredictable}
	\EE{f_i \X \vert \prevF{\X} } = 0 \enspace .
	\end{equation}
\end{lemma}
\begin{proof}
	Assume that \eqref{eq:unpredictable} holds and suppose to the contrary that the $i$th projection is non-trivial and redundant, so that $f_i(X)=h(f_{i-1}(X),\ldots,f_1(X))$ for some function $h$. According to the orthogonality property of the conditional expectation,
	\begin{equation}
	\E [ ( f_i(X)- \E[ f_i(X) | \prevF{\X} ] ) \, g( \prevF{\X} ) ] = 0
	\end{equation}
	for every function $g$. Substituting~\eqref{eq:unpredictable}, this property implies that
	\begin{equation}\label{eq:strictCondition}
	\EE{f_i \X \, g(\prevF{\X})} = 0, \quad\forall g \enspace .
	\end{equation}
	Therefore, in particular, for $g \equiv h$ we get that $\mathbb{E} [f_i^2 \X]=0$, contradicting our assumption that $f_i \X$ is non-trivial.
\end{proof}

Notice that by enforcing unpredictability, we in fact restrict each projection to be uncorrelated with \emph{any function} of the previous projections (see \eqref{eq:strictCondition}). This constraint is much stronger than the original zero correlation constraint \eqref{eq:orthogonality}.

\section{Algorithm} \label{sec:Unpredictability}

The unpredictability condition (\ref{eq:unpredictable}) is in fact an infinite set (a continuum) of constraints, as it restricts the conditional expectation of $f_i(X)$ to be zero, given every possible value that the previous projections $\{f_{j}\X\}_{j<i}$ may take. To obtain a practical method, we propose to enforce these restrictions only at the sample embedding points, leading to a discrete set of $N$ constraints
\begin{equation}\label{eq:discreteConstraints}
\EE{f_i \X \vert \{f_{j} \X = f_{j}(\boldsymbol{x}_n)\}_{j<i}} = 0 ,\quad n=1,\ldots, N \enspace .
\end{equation}
 These $N$ conditional expectations can be approximated using a kernel smoother matrix $\boldsymbol{P}_i \in \mathbb{R}^{N \times N}$ (e.g.\ the Nadaraya-Watson estimator \cite{nadaraya1964estimating,watson1964smooth}) for regressing $\f{i}$ against $\f{i-1},\ldots,\f{1}$, so that the $n$th entry of the vector $\boldsymbol{P}_i\f{i}$ constitutes an approximation to the $n$th conditional expectation in \eqref{eq:discreteConstraints},
\begin{equation}
\left[\boldsymbol{P}_i\f{i}\right]_n \approx \EE{f_i \X \vert \{f_{j} \X = f_{j}(\boldsymbol{x}_n)\}_{j<i}} \enspace .
\end{equation}
We therefore propose to replace the zero-correlation constraints $\f{i}^T\f{j}=0$ in~\eqref{prob:dimReduction}, by the unpredictability restrictions $\boldsymbol{P}_i\f{i} = \boldsymbol{0}$. Our proposed redundancy-avoiding version of the spectral dimensionality reduction problem \eqref{prob:dimReduction} is thus
\begin{equation}\label{prob:dimReductionUnpredictable}
\begin{aligned}
& \max_{\f{i}}
& & \f{i}^T \bK \f{i} \\
& \hspace{0.085cm} \mbox{s.t.}
& & \boldsymbol{1}^T \f{i}=0\\
&&& \f{i}^T \f{i} = 1\\
&&& \boldsymbol{P}_i \f{i} = \boldsymbol{0}, \quad \forall i>1 \enspace .
\end{aligned}
\end{equation}

Note that in the continuous domain, the conditional expectation operator has a non-empty null space. However, this property is usually not maintained by non-parametric sample approximations, like kernel regressors. As a result, the matrix $\boldsymbol{P}_i$ will typically be only \emph{approximately} low-rank. Figure \ref{fig:SVdecay} shows a representative example, where $0.1\%$ of the singular values account for over $99.9\%$ of the Frobenius norm. To ensure that $\boldsymbol{P}_i$ is strictly low-rank (so that the constraint $\boldsymbol{P}_i\f{i} = \boldsymbol{0}$  is not an empty set), we truncate its negligible singular values.

\begin{figure}[!t]
	\begin{center}
		\includegraphics[width=0.45\linewidth]{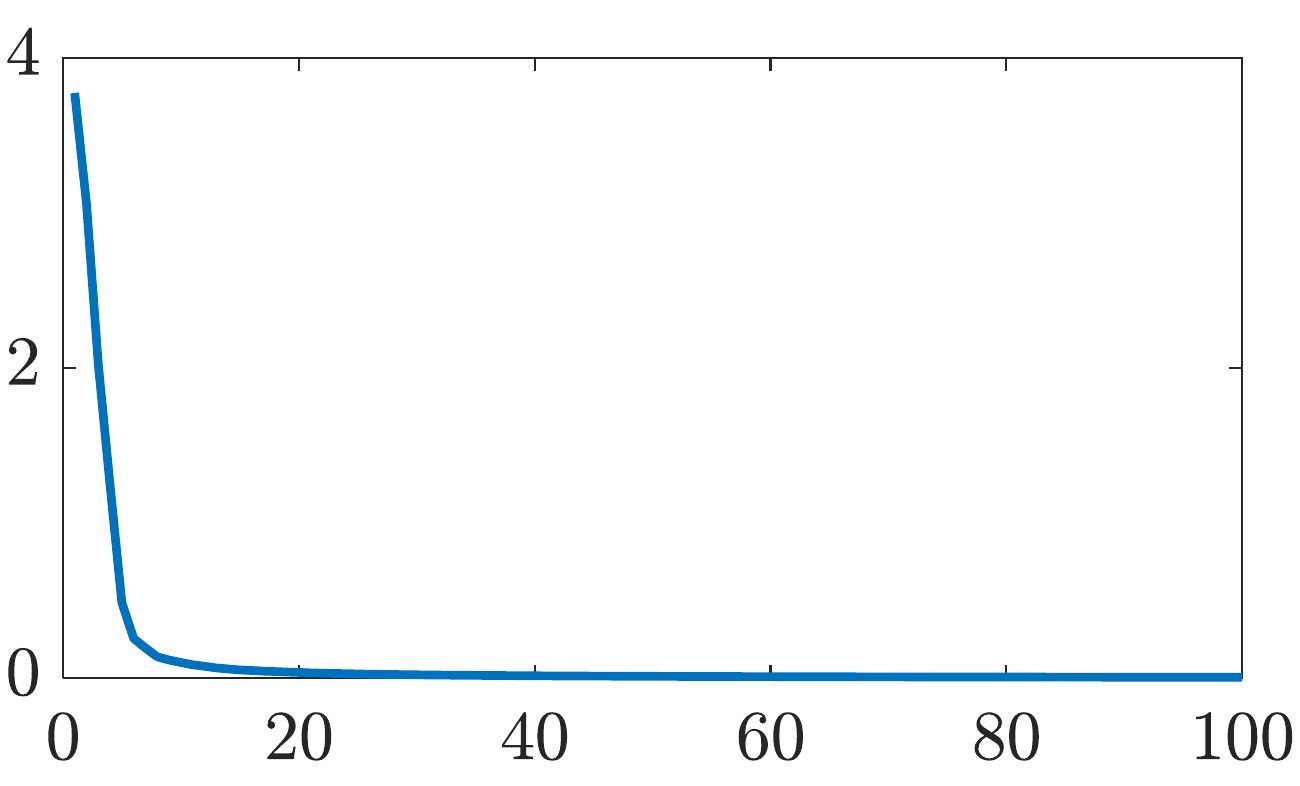}
	\end{center}
	\caption{Top 100 of 15K singular values of the matrix $\boldsymbol{P}_2$ in the MNIST experiment of Fig. \ref{fig:projRegression}. The matrix is very close to being low-rank: $0.1\%$ of its singular values account for over $99.9\%$ of its Frobenius norm.}
	\label{fig:SVdecay}
\end{figure}

The solution to problem \eqref{prob:dimReductionUnpredictable} is no longer given by the spectral decomposition of $\bK$. However, it can be brought into a convenient form by using the following lemma\footnote{Note that this lemma holds true only for \emph{maximization} problems.} (see proof in supplementary material).
\begin{lemma}\label{lem:solveUnpredictable}
	Denote the compact SVD of $\boldsymbol{P}_i$ by $\boldsymbol{U}_i \boldsymbol{D}_i \boldsymbol{V}_i^T$. Then the vectors $\f{1},\ldots,\f{d}$ which optimize Problem~\eqref{prob:dimReductionUnpredictable}, also optimize
	\begin{equation}\label{prob:dimReductionModified}
	\begin{aligned}
	& \max_{\f{i}}
	& & \f{i}^T \tilde{\bK}_i \f{i} \\
	& \hspace{0.085cm} \normalfont \mbox{s.t.}
	& & \boldsymbol{1}^T \f{i}=0\\
	&&& \f{i}^T \f{i} = 1 \enspace ,
	\end{aligned}
	\end{equation}
	where $\tilde{\bK}_i = ( \boldsymbol{I}-\boldsymbol{V}_i\boldsymbol{V}_i^T ) \bK ( \boldsymbol{I}-\boldsymbol{V}_i\boldsymbol{V}_i^T )$ and $\boldsymbol{V}_1=\boldsymbol{0}$.
\end{lemma}
From this lemma, it becomes clear that $\f{i}$ is precisely the top eigenvector of $\tilde{\bK}_i$. This implies that we can determine the non-redundant projections sequentially. In the $i$th step, we first modify the kernel $\bK$ according to the previous projections $\f{i-1},\ldots,\f{1}$ to obtain $\tilde{\bK}_i$. Then, we compute its top eigenvector to obtain projection $\f{i}$. 
This is summarized in Alg.~\ref{alg:NRDR}, where for concreteness, we chose $\boldsymbol{P}_i$ to be the Nadaraya-Watson smoother with a Gaussian-kernel.

\begin{algorithm}[!t]
	\caption{Non-redundant dimensionality reduction.}\label{alg:NRDR}
	\renewcommand{\algorithmicrequire}{\textbf{Input:}}
	\renewcommand{\algorithmicensure}{\textbf{Output:}}
	\begin{algorithmic}[1]
		\REQUIRE{High-dimensional data points $\boldsymbol{x}_n \in \mathbb{R}^{D}$.}
		\ENSURE{Embeddings $\f{i} = ( f_i(\boldsymbol{x}_1) , \cdots , f_i(\boldsymbol{x}_N) )^T$.}
		\STATE{Construct the kernel matrix $\bK$ as in the original algorithm (e.g.\ LLE, LEM, Isomap, etc.).}
		\STATE{If the original algorithm \emph{minimizes} the objective of (\ref{prob:dimReduction}) (e.g.\ LLE, LEM), then set $\bK \leftarrow \lambda_{\max}\boldsymbol{I}-\bK$.}
		
		\STATE{Assign the top (non-trivial)  eigen-vector of $\bK$ to $\f{1}$.}
		
		\FOR{$i = 2,\ldots,d$}		
		\STATE{Construct smoothing matrix
			\begin{align*}
			\left[\boldsymbol{P}_i\right]_{j,k} &\leftarrow \exp \left\{ - \frac{\sum_{\ell=1}^{i-1} \left( f_\ell \left( x_j \right) - f_\ell \left( x_k \right) \right)^2}{2 h^2}  \right\} \enspace , \\
			\left[\boldsymbol{P}_i\right]_{j,k} &\leftarrow \frac{\left[\boldsymbol{P}_i\right]_{j,k}}{\sum_{n=1}^{N}\left[\boldsymbol{P}_i\right]_{j,n}} \enspace .
			\end{align*}}		
		\STATE{Compute $\boldsymbol{V}_i \in \mathbb{R}^{N \times r}$, the top $r$ right singular vectors of $\boldsymbol{P}_i$ accounting for all non-negligible singular values.}		
		\STATE{Form the modified kernel matrix
			\begin{align*} \tilde{\bK}_i \leftarrow \left( \boldsymbol{I}-\boldsymbol{V}_i\boldsymbol{V}_i^T \right) \bK \left( \boldsymbol{I}-\boldsymbol{V}_i\boldsymbol{V}_i^T \right) \enspace .
			\end{align*}}
		\STATE{Assign the top eigen-vector of $\tilde{\bK}_i$ to $\f{i}$.}
		\ENDFOR	
	\end{algorithmic}
\end{algorithm}

\subsection{Efficient Implementation}
In several spectral dimensionality reduction algorithms (e.g.\ LLE, LEM) the kernel matrix $\bK$ is sparse, making them fit for large data sets in terms of memory and computational requirements. However, our modified kernel matrices $\tilde{\bK}_i$ are generally not sparse. To retain some of the efficiency of the original algorithms, we make two adjustments to Alg.~\ref{alg:NRDR}. First, in step~5 of the algorithm, we construct a \emph{sparse} smoothing matrix $\boldsymbol{P}_i$, by using only the $k$ nearest neighbors of each sample. This reduces the memory required to store $\boldsymbol{P}_i$ and also enables efficient computation of its top (non-negligible) singular vectors $\boldsymbol{V}_i$ (step~6). Second, we use the fast method of \cite{halko2011finding} to compute the top eigenvector of $\tilde{\bK}_i$ (step~8). Besides speed, this has the advantage that we never need to explicitly form the dense matrix $\tilde{\bK}_i$ (step~7). Indeed, each iteration of \cite{halko2011finding} involves multiplication by $\tilde{\bK}_i$, which can be broken into multiplications by $\boldsymbol{V}_i$, $\boldsymbol{V}_i^T$, and $\bK$. Therefore, we only have to store $\bK$, which is sparse, and $\boldsymbol{V}_i$, which is $N\times r$ with $r \ll N$.

It should be noted, however, that the effect of the sparsity of $\boldsymbol{P}_i$ on the running time and memory use, is somewhat more modest than could be expected. This is because the sparser $\boldsymbol{P}_i$ is, the slower its singular values decay, and thus the larger its rank. Thus, a sparser $\boldsymbol{P}_i$ requires computation of more singular vectors, which also slows the eigen-decomposition of $\tilde{\bK}_i$ (as $\boldsymbol{V}_i$ has more columns).

\subsection{Relation to Independent Component Analysis (ICA)}
Our method may seem similar to ICA \cite{jutten1991blind,hyvarinen1997fast}, however, they are quite distinct. First, the ICA \emph{objective} is independence (without preservation of geometrical structure), while in our method the objective is to preserve geometric structure subject to a statistical constraint on the embedding coordinates. Second, \emph{non-linear} ICA is an under-determined problem, making it necessary to impose assumptions or to restrict the class of non-linear functions \cite{hyvarinen1999nonlinear,singer2008non}. Finally, independence is a stronger constraint than unpredictability, and would thus narrow the set of possible solutions. This is while, as we saw, unpredictability is enough for avoiding redundancy.

\section{Experiments} \label{sec:Experiments}
We tested our non-redundant algorithm on three high-dimensional data sets. In all our experiments, we report results with the Nadaraya-Watson smoother \cite{nadaraya1964estimating,watson1964smooth}, as specified in Alg.~\ref{alg:NRDR}.
We also experimented with a locally linear smoother and did not observe a significant difference. The kernel smoother bandwidth $h$  was set adaptively: for computing $\boldsymbol{P}_i$, we took $h = \alpha (\sum_{j=1}^{i-1}\tfrac{1}{N}\|\boldsymbol{f}_j\|^2)^{1/2}$, where the parameter $\alpha \in [0.1,0.6]$ was chosen using a tune set in the classification task and manually in the visualization tasks. Singular vectors corresponding to singular values smaller than $3\%$ of the largest singular value were truncated. We used the largest number of nearest neighbors such that $\boldsymbol{P}_i$ could still be stored in memory (10K in our case).

\subsection{Artificial Head Images}

The artificial head image dataset \cite{tenenbaum2000global} is a popular test bed for manifold learning techniques. It contains $64\times64$ computer-rendered images of a head, with varying vertical and horizontal camera positions (denoted by $\theta$ and $\phi$) and lighting directions (denoted by $\psi$). Since each of the parameters $(\theta,\phi,\psi)$ varies significantly across this data set, most spectral methods manage to non-redundantly extract those parameters with the first three projections.

Here, to make the representation learning task more challenging, we chose a $257$ subset of the original data set, corresponding to the reduced parameter range $\theta\in[-75^\circ,75^\circ]$, $\phi\in[-8^\circ,8^\circ]$, $\psi\in[105^\circ,175^\circ]$. Figures~\ref{fig:facesPlot}(a),(c) visualize the projections extracted by LEM and LTSA in this case. As can be seen, both algorithms produce redundant representations, as their second projection is a deterministic function of the first. When incorporating our unpredictability constraints, we are able to avoid this repetition, as evident from Figs.~\ref{fig:facesPlot}(b),(d). Indeed, in our method, the second projection clearly carries additional information w.r.t.\@ the first.

\begin{figure}[!t]
	\begin{center}
		\includegraphics[width=\linewidth]{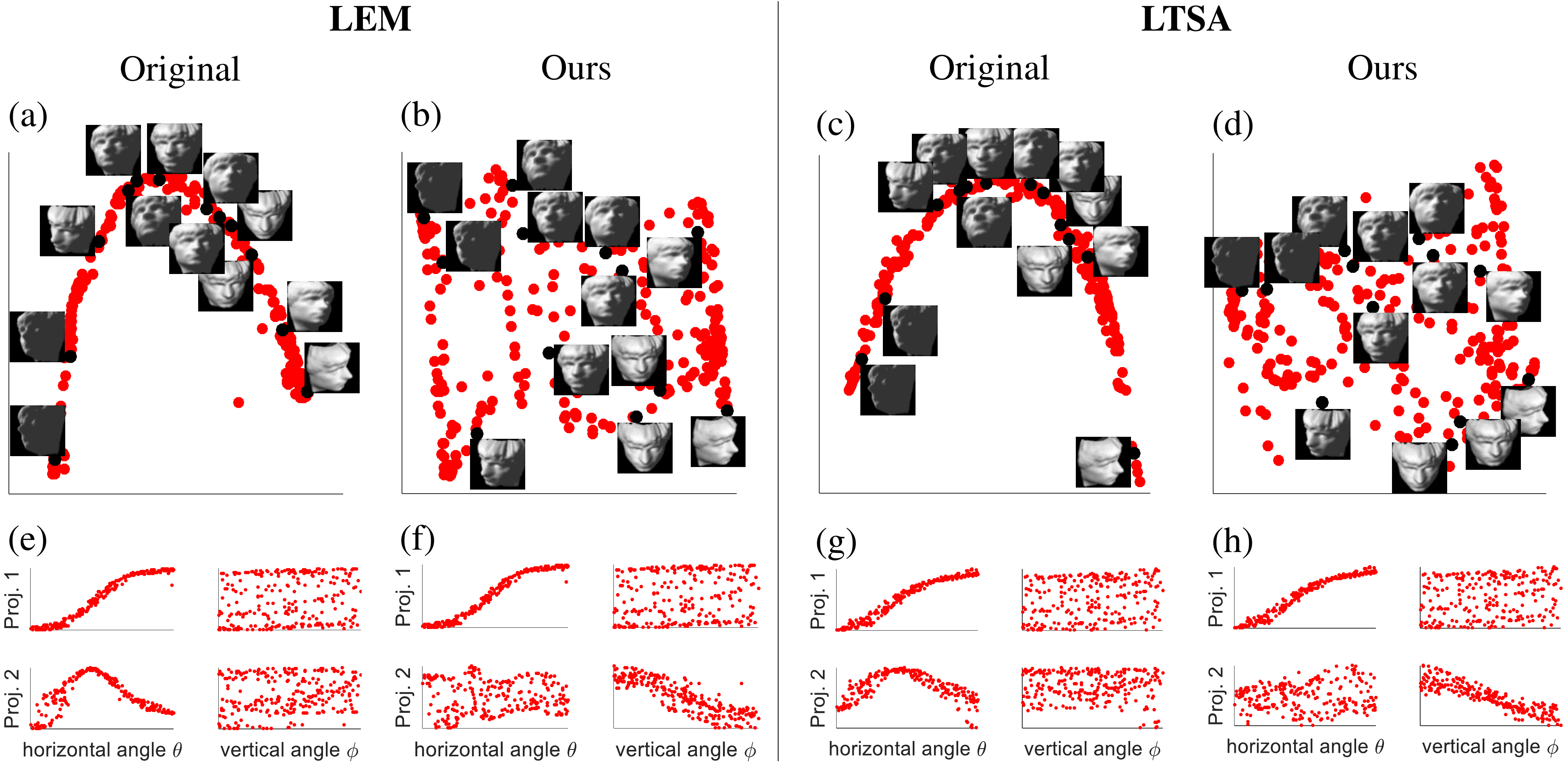}
	\end{center}
	\caption{Two-dimensional embeddings of computer rendered face images with varying vertical angles, horizontal angles, and lighting directions. (a) The original LEM method. (b) Our non-redundant LEM. (c) The original LTSA method. (d) Our non-redundant LTSA. The original algorithms produce redundant representations, as their second coordinate is a function of the first. In our method, the second coordinate clearly carries additional information w.r.t.\@ first, and therefore our representations are non-redundant. (e)-(h) The first two projections of the head images vs.\@ the horizontal and vertical angles $(\theta,\phi)$ of the heads. The two projections extracted by the original algorithms are \emph{both} correlated only with the horizontal angle $\theta$. In our non-redundant algorithms, on the other hand, the second projection is correlated with the vertical angle $\phi$.}
	\label{fig:facesPlot}
\end{figure}

To analyze what the projections capture, we plot in Fig.~\ref{fig:facesPlot}(e)-(h) each of the embedding coordinates vs.\@ the horizontal and vertical camera positions. From Figs.~\ref{fig:facesPlot}(e),(g) it becomes obvious that in the original algorithms, Projections~1 and~2 are both correlated only with the horizontal angle $\theta$. In our approach, on the other hand, Projection~1 captures the horizontal angle $\theta$ while Projection~2 reveals the vertical angle~$\phi$ (see Figs.~\ref{fig:facesPlot}(f),(h)).

\subsection{Image Patch Representation}

To visualize the effect of non-redundancy in low-level vision tasks, we extracted all $7 \times 7$ patches with 3 pixel overlap from an image (taken from \cite{rubinstein2010comparative}), and learned a three dimensional representation using Isomap and using our non-redundant version of Isomap. Figure~\ref{fig:patchesExperiment} visualizes the first three projections by coloring each pixel according to the embedding value of its surrounding patch. Observe that in the original algorithm, the first projection captures brightness, the second redundantly captures brightness once more, and the third captures mainly vertical edges with some brightness attributes still remaining (e.g.\ the sky, the left poolside). In contrast, in our algorithm, the second and third projections capture the vertical \emph{and horizontal} edges (without redundantly capturing brightness multiple times), thus providing additional information. The redundancy of the $2$nd Isomap  projection is clearly seen in the scatter plot of the $2$nd projection vs. the $1$st. With our non-redundant algorithm, the $2$nd projection is clearly not a function of the $1$st, and thus captures new informative features. 

\begin{figure}[!t]
	\begin{center}
		\includegraphics[width=\linewidth]{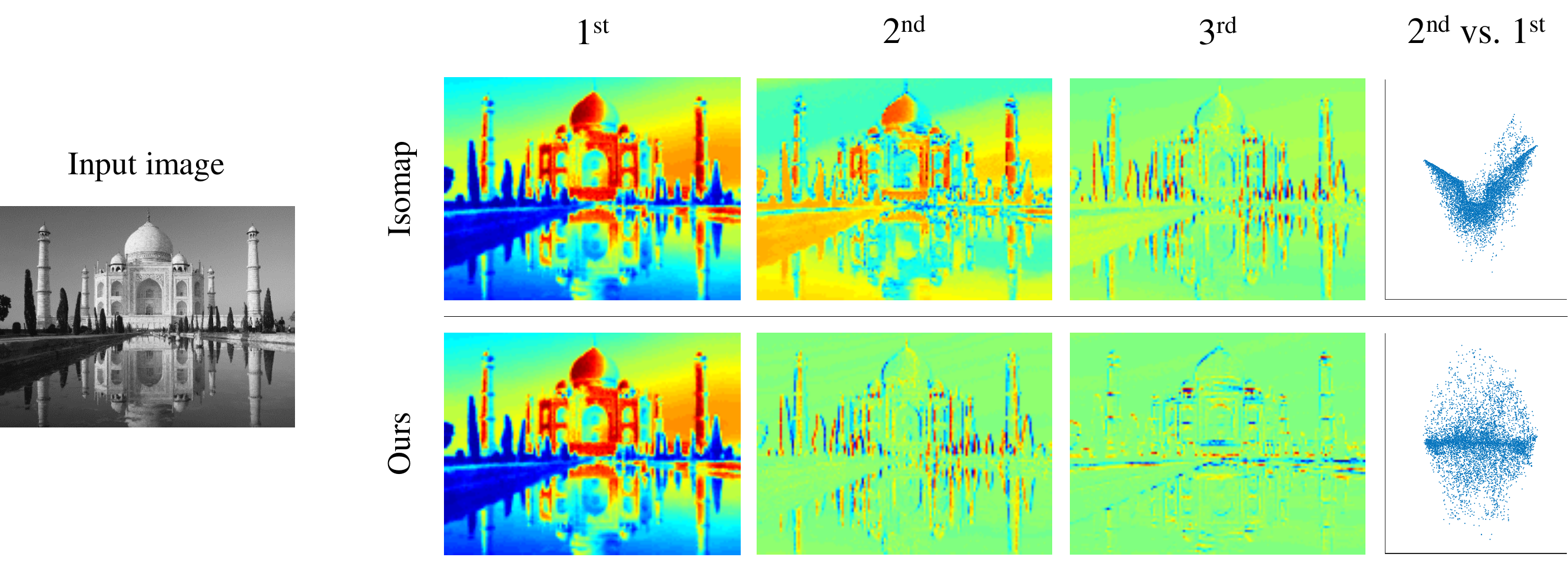}
	\end{center}
	\caption{Three-dimensional embedding of all $7\times7$ patches with a 3 pixel overlap, obtained with Isomap and with our non-redundant version of Isomap. Each pixels is colored according to the projection of its surrounding patch. In both methods, the first projection captures brightness. However, the original Isomap redundantly captures brightness-related features again with the second projection, and captures vertical edges only with the third projection. In contrast, our non-redundant version captures vertical \emph{and horizontal} edges with the second and third projections. The scatter plot reveals that in the original Isomap, the $2nd$ projection is a function of the $1$st, while in ours it is not.}
	\label{fig:patchesExperiment}
\end{figure}

Notice that the brightness and gradient features are linear functions of the input patches. Thus, our extracted 3D manifold is in fact linear and would be also correctly revealed by linear methods, such as PCA (not shown). Nevertheless, Isomap which is a nonlinear method, fails to extract this linear manifold due to redundancy (similarly to Fig. \ref{fig:strips}). In contrast, our non-redundant algorithm can reveal the underlying manifold regardless of its complexity.

\subsection{MNIST Handwritten Digits}\label{sec:MNIST}

In most practical applications, the ``correct'' parametrization of the data manifold is not as obvious as in the head experiment. One such example is the MNIST database~\cite{MNIST}, which contains $28 \times 28$ images of handwritten digits. In such settings, determining the quality of a low-dimensional representation can be done by measuring its impact on the performance in downstream tasks, like classification.

In the next experiment, we randomly chose a subset of 15K images from the MNIST data set, based on which we learned low-dimensional representations with LEM and with three modifications of LEM: (i)~the sequential regression technique (Sect.~\ref{sec:RelatedWork}), (ii)~the algorithm of Dsilva et. al.\ \cite{dsilva2015parsimonious}, and (iii)~our non-redundant method. We then split the data into 10K/2.5K/2.5K for training/tuning/testing and trained a third degree polynomial-kernel SVM \cite{chang2011libsvm} to classify the digits based on their low-dimensional representations. The SVM's soft margin parameter $c$ and kernel parameter $\gamma$ were tuned based on performance on the tune set (within the range $c \in [1,10]$, $\gamma \in [0.1,0.2]$). Table \ref{tab:MNIST} shows the classification error for various representation sizes. As can be seen, our non-redundant representation leads to the largest and most consistent decrease in the classification error.

\begin{table}
	\centering
	\caption{MNIST experiment classification errors $[\%]$.}
	\label{tab:MNIST}
	\begin{tabular}{c|cccc}
		\hline
		\multicolumn{5}{c}{\textbf{15K examples, all labeled}}\\
		\hline
		\makecell{\# of\\proj.} \hspace{0.1cm} & \hspace{0.1cm} \makecell{Laplacian\\eigenmaps}\hspace{0.1cm} & \makecell{Dsilva\\ et.~al.}\hspace{0.1cm} & \makecell{Sequential\\regression}\hspace{0.1cm} & Ours\\
		\hline
		3 & 17.6 & 17.6 & 17.3 & \textbf{12.0} \\
		5 & 8.8 & 8.8 & 14.4 & \textbf{7.6} \\
		7 & 6.9 & 6.9 & 14.2 & \textbf{6.0} \\
		9 & 6.5 & 6.5 & 14.2 & \textbf{5.6} \\
		11 & 6.0 & 5.4 & 13.8 & \textbf{5.0} \\
		\hline
	\end{tabular}
\quad\quad
	\begin{tabular}{c|cc}
		\hline
		\multicolumn{3}{c}{\textbf{15K examples, 300 labeled} }\\
		\hline
		\makecell{\# of\\proj.} \hspace{0.1cm} & \hspace{0.1cm} \makecell{Laplacian\\eigenmaps}\hspace{0.1cm} & Ours\hspace{0.1cm}\\
		\hline
		5 & 12.6 & 10.3 \\
		16 & 8.4 & {\color{red} 6.6} \\
		24 & {\color{red} 7.2} & 7.2 \\
		35 & 7.8 & 8.1 \\
		50 & 8.8 & 8.8 \\
		\hline
	\end{tabular}
\end{table}

To demonstrate the importance of \emph{compact} representations, particularly in the semi-supervised scenario, we repeated the experiment where only 300 of the examples are labeled for the SVM training (right pan of Table \ref{tab:MNIST}). Notice that the error reaches a minimum at 16\slash24 projections with our\slash LEM method, and then begins to rise as the representation dimension increases. This illustrates that unnecessarily large representations result in inferior performance in downstream tasks. Our method, which is designed to construct compact representations, achieves a lower minimal error ($6.6\%$ vs. $7.2\%$).

\section{Conclusions} \label{sec:Conclusions}
We presented a general approach for overcoming the redundancy phenomenon in spectral dimensionality reduction algorithms. As opposed to prior attempts, which fail in complex high-dimensional situations, our approach provably produces non-redundant representations. This is achieved by replacing the orthogonality constraints underlying spectral methods, by unpredictability constraints. Our solution reduces to applying a sequence of spectral decompositions, where in each step, the kernel matrix is modified according to the projections computed so far. Our experiments clearly illustrate the ability of our method to capture more informative compact representations of high-dimensional data.

\bibliographystyle{splncs03}
\bibliography{NRDRbib}

\end{document}